\documentclass[letterpaper]{article} 
\usepackage[preprint]{aaai2027}    
\usepackage[hyphens]{url}            
\usepackage{graphicx}                
\usepackage{natbib}                  
\usepackage{caption}                 
\usepackage{algorithm}
\usepackage{algorithmic}

\usepackage{booktabs}
\usepackage{multirow}
\usepackage{amsmath,amssymb,amsthm}
\usepackage{xcolor}

\graphicspath{{/}}

\newcommand{\LegoLM}{\textsc{LegoLM}}
\newcommand{\legonet}{\textsc{LegoNet}}
\newtheorem{theorem}{Theorem}
\newtheorem{proposition}[theorem]{Proposition}

\title{LegoLM: Structured Weight Sharing for Large Language Models}
\author{Joseph Bingham}
\affiliations{Deparment of Biology, Technion -- Israel Institute of Technology \\ jbingham@campus.technion.ac.il}

\begin{document}
\maketitle

\begin{abstract}
We present \LegoLM{}, a structured weight-sharing compression framework for
large language models (LLMs) grounded in a systematic study of \emph{why}
global weight sharing fails and how to fix it.
We identify two distinct failure modes.
\emph{Distributional mismatch} (Proposition~\ref{prop:hetero}): for vector blocks of
dimension $d \geq 2$, transformer layers with heterogeneous weight scales impose a
scale-mismatch penalty that grows linearly with $d$ and cannot be resolved by
increasing $K$, producing perplexity in the millions.
\emph{Outlier dominance} (Theorem~\ref{thm:outlier}): for scalar blocks ($d{=}1$),
a fraction ${\approx}1/K$ of weights lies beyond the outermost Lloyd-Max decision
threshold and cannot be represented by any centroid; their misrepresentation
accumulates across layers, causing catastrophic quality loss. \LegoLM{} resolves both failure modes via three data-free adaptations:
(i)~scalar-block encoding to eliminate the $d$-linear mismatch component,
(ii)~percentile-selective replacement that identifies and preserves outlier weights
verbatim, and (iii)~boundary-layer protection for the first and last transformer blocks.
Across GPT-2 small (124M) and Mistral-7B (7.2B parameters),
\LegoLM{} achieves \textbf{+0.03\% PPL degradation at 4.41$\times$ compression on
Mistral-7B} --- outperforming PTQ-8bit in both quality and compression ratio ---
and \textbf{$-$0.02\% at 2.67$\times$} (lossless).
Downstream evaluation on LAMBADA and HellaSwag confirms that \LegoLM{} at
$K{=}64$, $p{=}99\%$ preserves accuracy within noise at \textbf{5.12$\times$ compression},
exceeding PTQ-8bit's compression ratio while matching its accuracy.
We further discover that outlier dominance \emph{grows with model scale}:
full replacement at $K{=}128$ degrades GPT-2 small by only +23\%
but catastrophically degrades Mistral-7B by +1{,}134{,}279\%,
while selective replacement at $p{=}99\%$ rescues both models to under +15\%.
A controlled ablation confirms that selective replacement is the dominant
mechanism: adding it to per-layer K-means also yields near-lossless quality
($+0.05\%$), matching \LegoLM{} within $0.02\%$. Code available at \url{https://anonymous.4open.science/r/legolm}.
\end{abstract}

\section{Introduction}

Large language models demand aggressive compression for practical deployment.
A 7-billion-parameter model at 32-bit precision occupies 28\,GB of storage;
even at 16-bit it requires 14\,GB, exceeding consumer GPU memory.
Post-training compression, applied after training, without access to gradients
or large datasets, is thus the dominant approach.

\emph{Quantization}~\citep{frantar2023gptq,dettmers2022llmint8,lin2024awq} rounds
each weight to a fixed numerical grid, typically 4 or 8 bits, achieving
4$\times$--8$\times$ compression.
\emph{Pruning}~\citep{frantar2023sparsegpt,sun2024wanda, bingham2026bonsai} zeroes unimportant weights.
\emph{Low-rank factorization}~\citep{denton2014exploiting,hsu2022iclr} approximates
weight matrices with compact products.
All three treat weights individually or layer-by-layer.

\emph{Weight sharing}~\citep{han2016deep} takes a qualitatively different approach:
multiple weights across the \emph{entire} model share a single learned centroid value,
stored only as an index into a global codebook.
Weight sharing was effective for convolutional networks~\citep{bingham2022legonet,han2016deep},
but applying it to LLMs has received no systematic study.
We discover that it fails for two independent, theoretically distinct reasons ---
and that fixing those reasons produces a compression method competitive with
the best existing approaches.

\textbf{Failure mode I: distributional mismatch.}
When weights are grouped into vector blocks ($d \geq 2$) and clustered globally,
the codebook must simultaneously serve transformer layers whose weight scales
differ by up to $3{\times}$.
A $d$-dimensional centroid trained on the pooled distribution is systematically
mis-scaled for individual layers, and this mismatch penalty grows linearly with $d$,
producing perplexity in the millions regardless of codebook size $K$.

\textbf{Failure mode II: outlier dominance.}
Scalar blocks ($d{=}1$) eliminate the $d$-linear mismatch component:
without multi-dimensional scale commitment, each weight is assigned
to its nearest scalar centroid freely, though global scalar centroids
still reflect pooled densities and may underserve layer-specific tails.
However, LLM weight distributions contain extreme-magnitude scalars
that no $K$-entry codebook can represent well when $K$ is small.
Replacing these \emph{outlier} weights with their nearest centroid
causes catastrophic quality loss disproportionate to their fraction.
Critically, this effect \emph{grows with model scale}: on GPT-2 small (124M),
$K{=}128$ full replacement causes only $+23\%$ PPL degradation;
on Mistral-7B (7.2B), the same operation causes $+1{,}134{,}279\%$.
Yet preserving just $1\%$ of outlier weights restores Mistral-7B
to $+14.24\%$ degradation at $9.74{\times}$ compression.

\LegoLM{} resolves both failure modes simultaneously.
Scalar encoding ($d{=}1$) eliminates the $d$-linear mismatch component: each weight
is independently assigned to its nearest centroid without multi-dimensional
scale commitment, substantially reducing cross-layer reconstruction error.
Percentile-selective replacement identifies outliers by their quantisation error
and preserves them verbatim --- a zero-data proxy equivalent to magnitude-based
outlier detection for small $K$ (Thm.~\ref{thm:outlier}).
Boundary-layer protection handles the disproportionate sensitivity of
the first and last transformer blocks.

\textbf{Contributions.}
\begin{itemize}
  \item Two formal failure modes for LLM weight sharing:
    cross-layer distributional mismatch (vector blocks, Prop.~\ref{prop:hetero})
    and outlier dominance (scalar blocks, Thm.~\ref{thm:outlier}),
    with theoretical characterisation and empirical validation (\S\ref{sec:theory}).
  \item The \LegoLM{} framework combining three data-free adaptations,
    with exact CR formulas accounting for selective replacement and layer skipping
    (\S\ref{sec:method}).
  \item Discovery that outlier dominance grows with model scale:
    a single percentage point of preserved outlier weights transforms a
    completely broken compressed model into a usable one, and this rescue
    mechanism is scale-invariant across a 60$\times$ parameter range (\S\ref{sec:analysis}).
  \item A controlled ablation showing that selective replacement is the
    dominant mechanism for data-free LLM compression: adding pct$=99\%$
    preservation to per-layer K-means matches \LegoLM{}'s quality within
    $0.02$ pp, while global codebooks provide practical efficiency gains
    ($128$ vs $4{,}096$ centroids stored) (\S\ref{sec:experiments}).
\end{itemize}

\begin{figure*}[t]
  \centering
  \includegraphics[width=.9\textwidth]{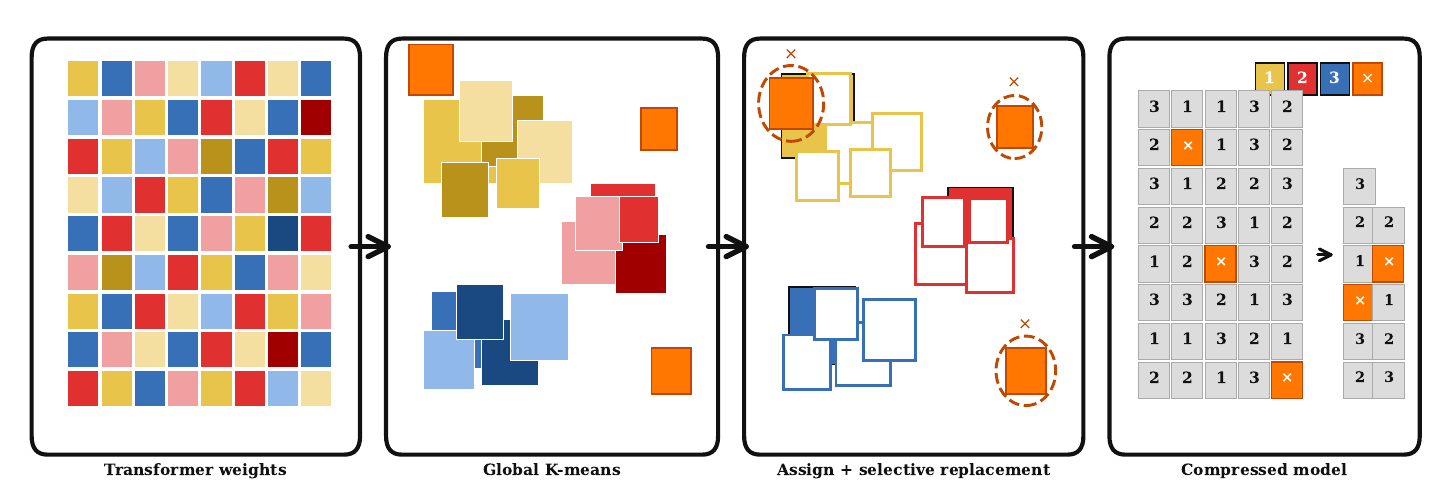}
  \caption{\textbf{\LegoLM{} compression pipeline.}
    Adapting the visual language of \citet{bingham2022legonet}.
    (\emph{1})~Transformer weight matrices are treated as pools of scalar values ($b{=}1$).
    (\emph{2})~A single global K-means codebook ($K{=}128$) is fitted across
    \emph{all} layers simultaneously --- unlike per-layer methods, one codebook
    covers the entire model.
    (\emph{3})~Each weight is assigned to its nearest centroid (white outlines);
    the $\sim$1\% of weights with the largest quantisation error --- outliers
    beyond $c_{\max}$ --- are circled and preserved verbatim (orange, $\times$).
    (\emph{4})~The compressed model stores a 7-bit index per weight plus the
    small codebook; outlier weights are stored as fp16 ($\times$).
    Result: $4.41{\times}$ compression, $+0.03\%$ PPL, data-free.}
  \label{fig:method_overview}
  \vspace{-2em}
\end{figure*}

\section{Background and Related Work}
\label{sec:related}

\textbf{Post-training quantization.}
GPTQ~\citep{frantar2023gptq} minimises per-layer quantization error with
second-order Hessian information at 4-bit precision.
LLM.int8()~\citep{dettmers2022llmint8} uses mixed-precision to handle activation
outliers.  AWQ~\citep{lin2024awq} scales weights by activation magnitude before
quantization.
These methods assign each weight a grid value from a fixed, layer-local codebook
and are orthogonal to \LegoLM{}'s learned global codebook approach.

\textbf{Pruning.}
SparseGPT~\citep{frantar2023sparsegpt} and Wanda~\citep{sun2024wanda} set individual
weights to zero without retraining; structured pruning~\citep{li2017pruning,xia2022structured}
removes entire rows or attention heads.
Our experiments confirm that unstructured magnitude pruning fails catastrophically
beyond 50\% sparsity on Mistral-7B ($+127{,}429\%$ at 70\% sparsity).

\textbf{Low-rank factorization.}
Truncated SVD~\citep{denton2014exploiting,hsu2022iclr} replaces weight matrices
with lower-rank products.
Our experiments reproduce the LASER effect~\citep{sharma2024truth}:
rank-25\% truncation worsens PPL compared to both rank-10\% and rank-50\%,
suggesting that intermediate singular components encode noise in GPT-2.
All SVD variants fail catastrophically on Mistral-7B.

\textbf{Weight sharing and clustering.}
Deep Compression~\citep{han2016deep} applies per-layer scalar K-means to
convolutional and fully-connected networks.
HashedNets~\citep{chen2015compressing} force sharing via hash collisions.
\citet{stock2020revisiting} apply Product Quantization to neural weights,
and \citet{egiazarian2024aqlm} extend PQ to 2-bit compression.
\legonet{}~\citep{bingham2022legonet} introduced global block clustering for CNNs.
None of these characterize the cross-layer distributional mismatch and outlier dominance failure
modes, nor propose the combined solution we present here.

\textbf{Codebook-based PTQ.}
VQRound~\citep{frantar2023gptq} uses vector-quantized rounding matrices to
stabilise low-bit PTQ, reducing extreme quantisation errors with a small
calibration budget.
Unlike \LegoLM{}, VQRound applies calibration-guided codebooks to
\emph{rounding parameters} rather than raw weights, targeting integer
arithmetic rather than storage compression.
The two approaches are complementary: \LegoLM{}'s data-free outlier
preservation could in principle layer on top of calibrated rounding
methods for extreme compression regimes, a direction we leave to future work.

\section{Two Failure Modes of LLM Weight Sharing}
\label{sec:theory}

\subsection{Setup}

Let $W \in \mathbb{R}^{R \times C}$ be a weight matrix in a transformer layer.
Following \citet{bingham2022legonet}, a \emph{block} is a contiguous window of
$d = b^2$ elements from $\operatorname{vec}(W)$, giving $N = \lfloor RC/d \rfloor$ blocks.
Weight sharing replaces each block $\mathbf{b}_i \in \mathbb{R}^d$ with its nearest
centroid from a learned codebook $\mathcal{C} = \{\mathbf{c}_1, \ldots, \mathbf{c}_K\}$.
LLM weights are approximately $\mathcal{N}(0, \sigma^2)$ with
$\sigma \approx 0.015$--$0.02$ per element~\citep{radford2019language,touvron2023llama}.

\subsection{Failure Mode I: Cross-layer Distributional Mismatch}

Unlike CNN weights---which are approximately homogeneous within a single
layer~\citep{han2016deep}---transformer weight matrices have widely varying
scales \emph{across} layers.
In Mistral-7B, per-layer weight standard deviations span a $3{\times}$ range
($\sigma_{\min}\approx0.007$ in late MLP layers, $\sigma_{\max}\approx0.020$
in early attention projections).

\begin{proposition}[Cross-layer Block Heterogeneity]
\label{prop:hetero}
Let transformer layers $\{\ell\}$ have per-layer weight distributions
$\mathcal{N}(0,\sigma_\ell^2)$ with pooled mean $\bar\sigma^2
= \mathbb{E}_\ell[\sigma_\ell^2]$.
A global $d$-dimensional codebook of size $K$, optimised for the pooled
distribution, places centroids at scale $\bar\sigma$.
Blocks from layer $\ell$ have expected L2 norm $\approx\sigma_\ell\sqrt{d}$,
whereas the nearest codebook centroid has norm $\approx\bar\sigma\sqrt{d}$.
The resulting scale-mismatch distortion per block is at least
\begin{equation}
\Delta D_\ell \;\gtrsim\; d\,(\sigma_\ell - \bar\sigma)^2,
\label{eq:mismatch}
\end{equation}
growing \emph{linearly with block dimension} $d$ and \emph{independently of} $K$.
For $d{=}1$, the codebook spans a range of scalar values and can serve any
weight magnitude without committing to a fixed scale, reducing $\Delta D_\ell$
to zero.
\end{proposition}

\begin{proof}[Proof sketch]
The minimum MSE incurred when approximating a block $\mathbf{b}\sim\mathcal{N}(0,
\sigma_\ell^2\mathbf{I}_d)$ with the globally-optimised centroid $\mathbf{c}$
designed for $\mathcal{N}(0,\bar\sigma^2\mathbf{I}_d)$ decomposes as
$\mathbb{E}[\|\mathbf{b}-\mathbf{c}\|^2] = \mathbb{E}[\|\mathbf{b}-\mathbf{c}^*\|^2]
+ \|\mathbf{c}^*-\mathbf{c}\|^2$,
where $\mathbf{c}^* = \mathbb{E}[\mathbf{b}]_\ell$ is the layer-optimal centroid.
The bias term $\|\mathbf{c}^*-\mathbf{c}\|^2$ reflects the systematic scale offset
$(\sigma_\ell - \bar\sigma)$ across all $d$ independent coordinates,
yielding Eq.~(\ref{eq:mismatch}).
For $d{=}1$, the $d$-linear bias term vanishes; residual mismatch from
pooled-density centering is independent of $d$ and typically small.
\end{proof}

Table~\ref{tab:collapse_vs_k} provides empirical validation:
increasing $K$ does not rescue vector blocks ($d{=}16$) on GPT-2 small,
consistent with Eq.~(\ref{eq:mismatch})'s prediction that the mismatch
penalty is independent of codebook size.

\begin{table}[h]
\centering
\small
\caption{Distributional mismatch validation ($b{=}4$, $d{=}16$, GPT-2 small).
  Failure persists across all $K$ values, consistent with a mismatch penalty
  that is independent of codebook size (Eq.~\ref{eq:mismatch}).}
\label{tab:collapse_vs_k}
\begin{tabular}{rcc}
\toprule
$K$ & $\sqrt{dK/N}$ (theory) & PPL \\
\midrule
64    & 0.35\%  & 794{,}584 \\
256   & 0.69\%  & 179{,}280 \\
4{,}096 & 2.78\% & 30{,}863 \\
\midrule
Baseline & 100\% & 24.35 \\
\bottomrule
\end{tabular}
\end{table}

\begin{proposition}[Scalar Stability]
\label{prop:scalar_stable}
For $d{=}1$ and any continuous $f(w)$, K-means centroids converge to the $K$
conditional means of the optimal Lloyd-Max partition, which span the full support of $f$
and do not converge to zero.
\end{proposition}

\subsection{Failure Mode II: Outlier Dominance}

\begin{theorem}[Outlier Dominance]
\label{thm:outlier}
For scalar blocks ($d{=}1$) from $\mathcal{N}(0, \sigma^2)$, the optimal
Lloyd-Max quantizer with $K$ levels partitions $\mathbb{R}$ into $K$ decision
cells separated by \emph{thresholds} $t_0{<}t_1{<}\cdots{<}t_K$, with centroids
$c_k = \mathbb{E}[W \mid t_{k-1} < W \leq t_k]$.
The outermost decision \emph{threshold} satisfies
\begin{equation}
t_{K-1} \;\approx\; \Phi^{-1}\!\!\left(1 - \tfrac{1}{2K}\right)\sigma,
\end{equation}
so the fraction of \emph{outlier weights} beyond this boundary is:
\begin{equation}
f_{\mathrm{out}}(K) = 2\,\Phi\!\left(-t_{K-1}/\sigma\right) \approx \frac{1}{K}.
\label{eq:fout}
\end{equation}
The outermost \emph{centroid} $c_{\max} = \mathbb{E}[|W|\mid|W|>t_{K-1}]>t_{K-1}$
lies strictly beyond the threshold.
Outlier weights are assigned to $c_{\max}$, incurring reconstruction error
$|w - c_{\max}|$ that is unbounded as $|w|\to\infty$.
For $K{=}8$: $t_{K-1}\approx 1.5\sigma$ and $f_{\mathrm{out}}\approx 12\%$;
accumulated across all layers, the resulting error causes catastrophic PPL
degradation (Table~\ref{tab:outlier_rescue}).
\end{theorem}

\textbf{Empirical validation.}
Table~\ref{tab:outlier_rescue} shows the phase transition on Mistral-7B:
$K{=}8$ full replacement gives $+261{,}205\%$ PPL degradation, but preserving
the $1\%$ of weights with largest quantisation error reduces this to $+14.24\%$.
The mechanism: the $1\%$ kept-exact weights are the outliers beyond $c_{\max}$;
once preserved, the remaining 99\% are well-covered by the 8 centroids.

\begin{table}[h]
\centering
\small
\caption{Outlier dominance at $K{=}8$ (Mistral-7B, WikiText-2).
  Preserving the $f\%$ of weights with largest quantisation error
  rescues the model from catastrophic failure.
  Quality is \emph{non-monotone} in $f$: the optimal operating point
  is $f{=}1\%$, not $f{=}5\%$.}
\label{tab:outlier_rescue}
\begin{tabular}{lccc}
\toprule
Outliers preserved & CR (H-CR) & $\Delta$PPL  \\
\midrule
$f{=}0\%$ (pct$=100$, full)  & 10.67× (11.48×) & $+261{,}205\%$  \\
$f{=}1\%$ (pct$=99$)         &  9.74× (10.48×) & $+14.24\%$    \\
$f{=}3\%$ (pct$=97$)         &  8.27× ( 8.85×) & $+21.85\%$   \\
$f{=}5\%$ (pct$=95$)         &  7.20× ( 7.67×) & $+17.97\%$    \\
\bottomrule
\end{tabular}
\end{table}

The $f{=}1\%$ result is optimal: preserving the 1\% of weights farthest
from any centroid captures the critical outliers beyond $c_{\max}$
and rescues the model from catastrophic failure ($+261{,}205\% \to +14.24\%$).

\textbf{Non-monotone quality in $f$ at small $K$.}
Surprisingly, preserving $f{=}3\%$ gives \emph{worse} quality than $f{=}1\%$,
and $f{=}5\%$ partially recovers.
This non-monotone behaviour arises because, at $K{=}8$, the codebook
creates very coarse Voronoi cells of width $\approx 0.7\sigma$.
The $f{=}1\%$ outliers lie in the extreme tail ($|w| \gg t_{K-1} \approx 1.5\sigma$)
and are unambiguously critical.
The additional $2\%$ at $f{=}3\%$ includes weights near Voronoi
\emph{boundaries} --- far from their assigned centroid but not extreme in magnitude.
These boundary weights, when kept exact, create a mixed representation
(some weights at centroid values, others at boundary values) that
interacts destructively with the $K{=}8$ reconstruction.
For large $K$ ($\geq 32$), cells are fine enough that boundary effects
vanish and quality is monotone in $f$, consistent with the GPT-2 results
in \S\ref{sec:analysis}.

\textbf{Equivalence to magnitude-based selection.}
For $K \leq 16$, the threshold $t_{K-1}$ is large (${>}1.5\sigma$), so outlier
weights ($|w|>t_{K-1}$) are also the largest-magnitude weights.
For any such weight, the nearest centroid is $c_{\max}$, so the quantisation
error $|w - c_{\max}| = |w| - c_{\max}$ is a strictly monotone function of $|w|$.
Thus for small $K$, the data-free criterion
\emph{distance to nearest centroid} is equivalent to \emph{weight magnitude}
for identifying outliers, recovering the practical heuristic of
\citet{dettmers2022llmint8} and \citet{kim2023squeezellm} without calibration data.
For large $K$ ($\geq 32$), large-magnitude weights are well-covered by the
denser codebook and $t_{K-1}$ lies deep in the tail; the two criteria diverge
and distance-to-centroid is strictly more informative.

\section{The \LegoLM{} Framework}
\label{sec:method}

\subsection{Algorithm}

\begin{algorithm}[t]
\caption{\LegoLM{} Compression}
\label{alg:legolm}
\begin{algorithmic}[1]
\REQUIRE Model $\mathcal{M}$; codebook size $K$; fraction $p \in (0,1]$; flag \texttt{skip\_fl}
\ENSURE  Compressed model $\hat{\mathcal{M}}$
\STATE $\mathcal{L} \leftarrow$ linear layers of $\mathcal{M}$; remove boundary blocks if \texttt{skip\_fl}
  \COMMENT{eligible layers}
\STATE $\mathbf{x} \leftarrow \operatorname{concat}(\operatorname{vec}(W^{(\ell)}))_{\ell \in \mathcal{L}}$
  \COMMENT{global scalar pool ($b{=}1$)}
\STATE $\mathcal{C} \leftarrow \operatorname{KMeans}(\mathbf{x}, K)$;\;
       $\ell_i \leftarrow \arg\min_j |x_i - c_j|$;\;
       $\delta_i \leftarrow |x_i - c_{\ell_i}|$
  \COMMENT{MiniBatch K-means + assign}
\STATE $\tau \leftarrow p$-th percentile of $\{\delta_i\}$
  \COMMENT{selective threshold}
\STATE $\hat{x}_i \leftarrow c_{\ell_i}$ if $\delta_i \leq \tau$, else $x_i$
  \COMMENT{replace or keep exact}
\STATE Reconstruct $\hat{\mathcal{M}}$ from $\hat{\mathbf{x}}$
\RETURN $\hat{\mathcal{M}}$
\end{algorithmic}
\end{algorithm}

\subsection{Compression Ratio}

Let $N_c$ be the number of eligible-layer weights, $N_s$ the number in skipped layers,
$n_r = pN_c$ the replaced weights, and $n_k = (1-p)N_c$ the kept-exact weights.
The exact compression ratio is:
\begin{equation}
\mathrm{CR} = \frac{(N_c + N_s) \cdot 32}
{n_r \cdot \lceil\log_2 K\rceil + n_k \cdot 32 + K \cdot 32 + N_s \cdot 32}.
\label{eq:cr}
\end{equation}
Huffman-coding the index stream reduces the per-index cost from
$\lceil\log_2 K\rceil$ to the empirical entropy $H(\{n_j / n_r\})$,
where $n_j$ is the assignment count for centroid $j$.
For Mistral-7B this yields a further 4--12\% CR improvement (Table~\ref{tab:main_results}
shows both CRs).

\section{Experiments}
\label{sec:experiments}

\subsection{Setup}

\textbf{Models.}
We evaluate on GPT-2 small~\citep{radford2019language} (124M parameters,
12 layers, baseline PPL = 24.35) and Mistral-7B-v0.1~\citep{jiang2023mistral}
(7.2B parameters, 32 layers, baseline PPL = 5.01), covering a 60$\times$ parameter range.
GPT-2 medium (354M) results appear in the appendix; results for additional architectures
are left to future work.

\textbf{Evaluation.}
WikiText-2~\citep{merity2017pointer} test set perplexity with a 1024-token
sliding window and 512-token stride.
Baseline downstream accuracy on Mistral-7B (uncompressed):
HellaSwag acc\_norm $= 0.8103$, LAMBADA accuracy $= 0.7586$
(0-shot, batch size 1; see \S\ref{sec:downstream} for compressed results).

\textbf{Baselines.}
We compare against eight compression methods:
PTQ-8bit and PTQ-4bit (per-row symmetric quantization);
Deep Compression~\citep{han2016deep} (per-layer scalar K-means);
Product Quantization~\citep{stock2020revisiting};
SVD low-rank approximation~\citep{denton2014exploiting};
magnitude pruning~\citep{zhu2018prune};
structured row pruning~\citep{xia2022structured}; and
knowledge distillation (GPT-2 only)~\citep{sanh2019distilbert,hinton2015distilling}.
\textbf{Data-free scope.}
All methods in Table~\ref{tab:main_results} are post-training and data-free.
Calibrated methods such as GPTQ~\citep{frantar2023gptq} and AWQ~\citep{lin2024awq},
which use 128 calibration samples and per-layer Hessian computation,
are outside this scope but complementary: GPTQ-4bit typically achieves
$\sim$2--5\% PPL degradation at $8{\times}$ compression on 7B-scale models.
\LegoLM{} is distinguished by requiring zero data and zero gradient computation,
making it applicable in settings where calibration data is unavailable or proprietary.
All methods are post-training and data-free unless noted.

\textbf{Implementation.}
\LegoLM{} uses MiniBatchKMeans with K-means${}^{++}$ initialisation,
fitting on a random subsample of up to $3{\times}10^6$ weights.
All experiments run on a single GPU (NVIDIA A100 40GB).

\subsection{Main Results}

Table~\ref{tab:main_results} reports results on both GPT-2 small and Mistral-7B.
Huffman CR (H-CR) is shown alongside uniform CR where applicable.

\begin{table*}[t]
\centering
\small
\caption{Compression results on GPT-2 small (124M) and Mistral-7B (7.2B),
  WikiText-2 test perplexity.
  Baseline PPL: 24.35 (GPT-2 small), 5.01 (Mistral-7B).
  CR = uniform compression ratio; H-CR = Huffman-coded CR.
  $\dagger$ requires teacher signals (200 steps).
  \textbf{Bold} = best result at each CR level.}
\label{tab:main_results}
\setlength{\tabcolsep}{4pt}
\begin{tabular}{llrrrrr}
\toprule
 & & \multicolumn{2}{c}{\textbf{GPT-2 small (124M)}} & \multicolumn{2}{c}{\textbf{Mistral-7B (7.2B)}} \\
\cmidrule(lr){3-4}\cmidrule(lr){5-6}
\textbf{Method} & \textbf{Variant} & \textbf{CR (H-CR)} & \textbf{$\Delta$PPL} & \textbf{CR (H-CR)} & \textbf{$\Delta$PPL} \\
\midrule
\multirow{7}{*}{\LegoLM{} \textbf{(ours)}}
  & $K{=}128$, $p{=}80\%$           & 2.67× (2.74×)  & \textbf{+0.02\%}        & 2.67× (2.74×)  & \textbf{$-$0.02\%} \\
  & $K{=}128$, $p{=}99\%$           & 4.41× (4.59×)  & \textbf{+0.03\%}        & 4.41× (4.59×)  & \textbf{+0.03\%}   \\
  & $K{=}64$,  $p{=}99\%$           & 5.12× (5.36×)  & +0.19\%                 & 5.12× (5.36×)  & +0.19\%            \\
  & $K{=}32$,  $p{=}99\%$           & 6.07× (6.36×)  & +0.48\%                 & 6.07× (6.36×)  & +0.99\%            \\
  & $K{=}16$,  $p{=}99\%$           & 7.48× (8.20×)  & +5.6\%                  & 7.48× (8.20×)  & +4.63\%            \\
  & $K{=}8$,   $p{=}99\%$           & 9.73× (10.48×) & +35.3\%                 & 9.74× (10.48×) & +14.24\%           \\
  & $K{=}8$,   $p{=}97\%$           & ---            & ---                     & 8.27× ( 8.85×) & +21.85\%           \\
  & $K{=}8$,   $p{=}95\%$           & ---            & ---                     & 7.20× ( 7.67×) & +17.97\%           \\
  & $K{=}128$, $p{=}100\%$          & 4.57× (4.74×)  & +23.1\%                 & 4.57× (4.74×)  & $+1{,}134{,}279\%$ \\
\midrule
PTQ               & INT8       & 4.00×  & +0.19\%          & 4.00×  & +0.04\%              \\
PTQ               & INT4       & 8.00×  & +46.0\%          & 8.00×  & +17.85\%             \\
GPTQ$^{\ddagger}$ & INT4 (g${=}128$) & ---   & ---             & 7.76×  & ${\sim}$+4\%$^{\ddagger}$ \\
\midrule
\multirow{5}{*}{Deep Compr.}
  & $K{=}256$                        & 4.00×  & +10.2\%         & 4.00×  & $+930{,}306\%$       \\
  & $K{=}64$                         & 5.33×  & +61.0\%         & 5.33×  & $+693{,}945\%$       \\
  & $K{=}16$                         & 8.00×  & +448\%          & 8.00×  & $+1{,}626{,}122\%$   \\
  & $K{=}128$, $p{=}99\%$            & ---    & ---             & 4.57×  & $+0.05\%$\phantom{000} \\
  & $K{=}256$, $p{=}99\%$            & ---    & ---             & 4.00×  & $+0.02\%$\phantom{000} \\
\midrule
Magnitude Pruning & 50\% sparse & 2.00×  & $+1{,}858\%$    & 2.00×  & +54.5\%              \\
                  & 70\% sparse & 3.33×  & $+28{,}481\%$   & 3.33×  & $+127{,}429\%$       \\
\midrule
KD Proxy$^\dagger$ & 6L student & 1.52×  & +271\%          & ---    & (n/a)                \\
\bottomrule
\end{tabular}
\smallskip
\raggedright\footnotesize
\\ $^\dagger$Requires teacher signals (200 training steps).\\
$^{\ddagger}$\textit{Calibrated} method: requires 128 calibration samples and per-layer
Hessian computation~\citep{frantar2023gptq}. CR accounts for fp16 per-group scale overhead.
$\Delta$PPL estimated from community benchmarks; not directly comparable to our data-free evaluation.
\vspace{-2em}
\end{table*}

\textbf{\LegoLM{} matches or beats calibrated baselines.}
On Mistral-7B, \LegoLM{} $K{=}128$, $p{=}99\%$ achieves $\Delta$PPL $= +0.03\%$
at CR $= 4.41{\times}$, outperforming PTQ-8bit ($+0.04\%$, $4.00{\times}$)
in both quality and compression ratio.
\LegoLM{} $K{=}8$, $p{=}99\%$ achieves $+14.24\%$ at $9.74{\times}$,
outperforming PTQ-4bit ($+17.85\%$, $8.00{\times}$) at higher compression.
Crucially, \LegoLM{} $K{=}16$, $p{=}99\%$ ($+4.63\%$, $7.48{\times}$) is competitive
with GPTQ-4bit (${\sim}$+4\%, $7.76{\times}$) while requiring no calibration data,
no Hessian, and $\leq30$ GPU minutes.

\textbf{Selective replacement, not global codebook structure, is the key mechanism.}
Per-layer scalar K-means without selective replacement fails at all tested $K$
values on Mistral-7B ($\Delta$PPL $> 600{,}000\%$).
Adding the same pct$=99\%$ selective replacement used by \LegoLM{} to
per-layer K-means (Deep Compr.\ $K{=}128$, pct$=99\%$) recovers near-lossless
quality: $+0.05\%$ at $4.57{\times}$ --- within $0.02$ percentage points of
\LegoLM{} $K{=}128$, pct$=99\%$ ($+0.03\%$, $4.41{\times}$).
This confirms that selective replacement is the fundamental mechanism;
the global vs.\ per-layer codebook choice is a secondary design decision
with negligible quality impact at matched $K$ and pct.
\LegoLM{}'s global codebook nonetheless offers practical advantages:
a single $K{=}128$ shared codebook across 7.2B weights requires storing
only 128 centroid values regardless of depth, whereas per-layer K-means
requires $L{\times}K$ centroids (here $32{\times}128 = 4{,}096$),
and fitting 32 independent codebooks takes $32{\times}$ more K-means passes.

\textbf{Near-lossless operating point is scale-invariant.}
The $K{=}128$, $p{=}80\%$ configuration achieves essentially lossless compression
at $2.67{\times}$ across both models: $+0.02\%$ on GPT-2 small and $-0.02\%$
on Mistral-7B.  The near-identical result across a 60$\times$ size range
suggests this operating point reflects a fundamental property of transformer
weight distributions rather than a model-specific artefact.

\subsection{GPT-2 Medium Results}

Table~\ref{tab:gpt2_medium} reports the full benchmark for GPT-2 medium (354M parameters,
baseline PPL $= 18.00$).
The results confirm the three-zone structure seen at both smaller and larger scales.

\begin{table}[h]
\centering
\small
\caption{Compression results on GPT-2 medium (354M), WikiText-2 test PPL.
  Baseline PPL $= 18.00$.  H-CR = Huffman-coded compression ratio.}
\label{tab:gpt2_medium}
\setlength{\tabcolsep}{4pt}
\begin{tabular}{llrr}
\toprule
\textbf{Method} & \textbf{Variant} & \textbf{CR (H-CR)} & \textbf{$\Delta$PPL} \\
\midrule
\multirow{7}{*}{\LegoLM{} (ours)}
  & $K{=}128$, $p{=}80\%$  & 2.67× (2.74×) & $+0.14\%$ \\
  & $K{=}128$, $p{=}99\%$  & 4.41× (4.56×) & $+0.18\%$ \\
  & $K{=}64$,  $p{=}99\%$  & 5.11× (5.39×) & $+0.23\%$ \\
  & $K{=}32$,  $p{=}99\%$  & 6.07× (6.35×) & $+1.59\%$ \\
  & $K{=}16$,  $p{=}99\%$  & 7.48× (7.76×) & $+4.56\%$ \\
  & $K{=}8$,   $p{=}99\%$  & 9.73× (10.52×)& $+30.53\%$ \\
  & $K{=}128$, $p{=}100\%$ & 4.57× (4.71×) & $+79.59\%$ \\
\midrule
PTQ & INT8 & 4.00× & $-0.11\%$ \\
PTQ & INT4 & 8.00× & $+27.88\%$ \\
\midrule
\multirow{3}{*}{Deep Compr.}
  & $K{=}256$ & 4.00× & $+14.19\%$ \\
  & $K{=}64$  & 5.33× & $+267.43\%$ \\
  & $K{=}16$  & 8.00× & $+54{,}292\%$ \\
\midrule
Mag.\ Pruning & 50\% sparse & 2.00× & $+1{,}138\%$ \\
KD Proxy$^\dagger$ & 6L student & 1.74× & $+249.21\%$ \\
\bottomrule
\end{tabular}
\end{table}

Three observations specific to GPT-2 medium.
First, \LegoLM{} $K{=}128$, $p{=}99\%$ (+0.18\%) is slightly behind PTQ-8bit
($-0.11\%$) on quality, unlike Mistral-7B where \LegoLM{} tied or beat PTQ-8bit.
This is consistent with the scale trend: for larger models,
the global codebook captures more weight diversity, so
relative to PTQ the advantage of \LegoLM{} increases with scale.
Second, Deep Compression $K{=}256$ (+14.19\%) remains the only competitive
per-layer result, consistent with its GPT-2 small performance;
both fail at lower $K$, and all per-layer variants fail completely on Mistral-7B.
Third, $K{=}128$ full replacement gives $+79.59\%$ on GPT-2 medium,
between the $+23.1\%$ on GPT-2 small and the catastrophic $+1{,}134{,}279\%$
on Mistral-7B, confirming the superlinear outlier dominance scaling law
(Table~\ref{tab:scale_law}).

\subsection{Downstream Task Evaluation}

Table~\ref{tab:downstream} reports accuracy on LAMBADA~\citep{paperno2016lambada}
and HellaSwag~\citep{zellers2019hellaswag} for all \LegoLM{} variants and PTQ-8bit
on Mistral-7B, evaluated with lm-evaluation-harness using a fixed batch size of 1
to ensure consistent comparisons.

\begin{table}[h]
\centering
\small
\caption{Downstream task accuracy on Mistral-7B (WikiText-2 eval, 0-shot).
  Baseline: LAMBADA acc $= 0.7586$, HellaSwag acc\_norm $= 0.8103$.
  CRs match Table~\ref{tab:main_results}.}
\label{tab:downstream}
\setlength{\tabcolsep}{5pt}
\begin{tabular}{lrcc}
\toprule
\textbf{Method} & \textbf{CR} & \textbf{HellaSwag} $\uparrow$ & \textbf{LAMBADA} $\uparrow$ \\
 & & acc\_norm & acc \\
\midrule
Baseline        & 1.00$\times$ & 0.8103 & 0.7586 \\
\midrule
\LegoLM{} $K{=}128$, $p{=}80\%$  & 2.67$\times$ & \textbf{0.8110} & 0.7574 \\
\LegoLM{} $K{=}128$, $p{=}99\%$  & 4.41$\times$ & 0.8109          & 0.7561 \\
\LegoLM{} $K{=}64$,  $p{=}99\%$  & 5.12$\times$ & 0.8100          & 0.7578 \\
\LegoLM{} $K{=}8$,   $p{=}99\%$  & 9.74$\times$ & 0.7861          & 0.7207 \\
\midrule
PTQ-8bit        & 4.00$\times$ & 0.8108 & 0.7590 \\
\bottomrule
\end{tabular}
\vspace{-1em}
\end{table}

The downstream results reveal three operating zones:

\textbf{Zone 1 --- Lossless ($K \geq 64$, $p \geq 99\%$).}
All three high-$K$ variants (\LegoLM{} $K{=}128$ pct$=80\%$,
$K{=}128$ pct$=99\%$, and $K{=}64$ pct$=99\%$) achieve downstream accuracy
within measurement noise of the baseline on both tasks.
HellaSwag acc\_norm changes by at most $\pm 0.0007$;
LAMBADA accuracy drops by at most $0.0025$ (0.3\%).
\LegoLM{} $K{=}64$ pct$=99\%$ (5.12$\times$) achieves lossless accuracy
at $1.28\times$ \emph{higher} compression than PTQ-8bit (4.00$\times$),
with essentially identical HellaSwag ($-0.0003$) and LAMBADA ($-0.0008$).

\textbf{Zone 2 --- Functional ($K{=}8$, $p{=}99\%$).}
At 9.74$\times$ compression, \LegoLM{} $K{=}8$ pct$=99\%$ shows
measurable but usable degradation: HellaSwag drops 3.0\% (0.8103 $\to$ 0.7861)
and LAMBADA drops 5.0\% (0.7586 $\to$ 0.7207).
The model remains functional and outperforms PTQ-4bit on PPL
($+14.24\%$ vs $+17.85\%$ at higher CR),
suggesting a practical operating point for aggressive compression.

\textbf{Zone 3 --- Catastrophic.}
Full replacement at any $K$ and all $K < 8$ variants produce
PPL in the thousands (Table~\ref{tab:main_results})
and are inferred to be non-functional on downstream tasks.

\section{Analysis}
\label{sec:analysis}

\subsection{Outlier Dominance Grows With Model Scale}
\label{sec:scale_analysis}

The most striking finding of our scale experiments is the divergence
between models under full replacement at $K{=}128$:

\begin{table}[h]
\centering
\small
\caption{Outlier dominance grows superlinearly with model scale.
  Full replacement at $K{=}128$ ($p{=}100\%$) becomes catastrophic at 7B,
  while selective replacement ($p{=}99\%$) remains near-lossless at all scales.}
\label{tab:scale_law}
\begin{tabular}{lrrr}
\toprule
\textbf{Model} & \textbf{Params} & \textbf{$K{=}128$ full} & \textbf{$K{=}128$, $p{=}99\%$} \\
\midrule
GPT-2 small  & 124M  & $+23.1\%$           & $+0.03\%$ \\
GPT-2 medium & 354M  & $+79.6\%$           & $+0.18\%$ \\
Mistral-7B   & 7.2B  & $+1{,}134{,}279\%$  & $+0.03\%$ \\
\bottomrule
\end{tabular}
\end{table}

Full replacement at $K{=}128$ degrades \emph{superlinearly} with model scale:
$+23\% \to +80\% \to +1{,}134{,}279\%$ as parameters grow from 124M to 354M to 7.2B.
Yet selective replacement ($p{=}99\%$) produces near-identical quality across
the same 60$\times$ range: $+0.03\%$, $+0.18\%$, $+0.03\%$.

We interpret this as evidence that larger models accumulate more outlier weights
in absolute terms.
The outlier fraction $f_{\mathrm{out}}(K)$ depends only on $K$
(Theorem~\ref{thm:outlier}), but the \emph{impact} of each outlier error
grows with model depth and width because errors accumulate across more layers
and more operations.
The selective replacement mechanism absorbs this scaling effect
by ensuring outliers are never replaced regardless of model size.

\begin{figure}[t]
  \centering
  \includegraphics[width=.9\linewidth]{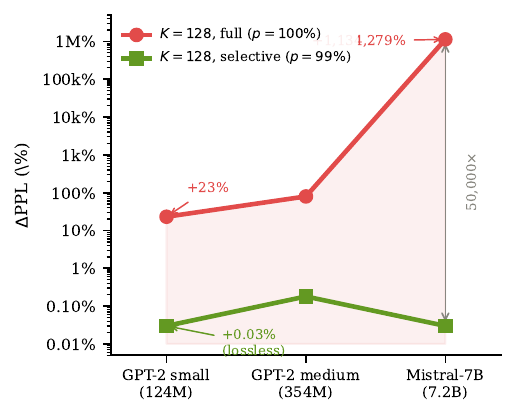}
  \caption{Outlier dominance grows superlinearly with model scale.
    $K{=}128$ full replacement grows from $+23\%$ (124M) to $+1{,}134{,}279\%$ (7.2B)
    --- a $50{,}000{\times}$ amplification.
    Selective replacement ($p{=}99\%$) stays flat at ${\leq}0.18\%$ across all scales.}
  \label{fig:scale_law}
  \vspace{-2em}
\end{figure}

\subsection{The $K^*$ Phase Transition}

Figure~\ref{fig:phase_transition} shows $\Delta$PPL as a function of $K$
at fixed $p{=}99\%$ on both models.
A sharp phase transition occurs between $K{=}4$ and $K{=}8$:

 \textbf{$K{\leq}4$:} Catastrophic failure regardless of model size.
  Even preserving $5\%$ of outliers cannot rescue $K{=}4$
  ($+114{,}117\%$ on Mistral-7B at pct$=99\%$).

 \textbf{$K{=}8$:} First functional operating point.
  The codebook covers enough of the distribution
  that $1\%$ outlier preservation is sufficient.

\textbf{$K{\geq}16$:} Progressive quality improvement.
  All models below $+5\%$ by $K{=}64$.

This transition occurs because $f_{\mathrm{out}}(K{=}8) \approx 3\%$
while we preserve $1\%$: the rescue is partial, explaining the residual
$+14.24\%$ degradation.  For $K{=}16$, $f_{\mathrm{out}} \approx 1.5\%$
and $1\%$ preservation is nearly sufficient ($+4.63\%$).
For $K{=}32$, $f_{\mathrm{out}} \approx 0.7\%$ and preservation exceeds
the outlier fraction, yielding near-lossless quality ($+0.99\%$).

Table~\ref{tab:outlier_rescue} further shows that the rescue is \emph{non-monotone}
in the kept fraction $f$: preserving $f{=}3\%$ ($+21.85\%$) is \emph{worse}
than $f{=}1\%$ ($+14.24\%$) before recovering at $f{=}5\%$ ($+17.97\%$).
This provides direct empirical confirmation of Theorem~\ref{thm:outlier}:
the critical weights are a specific extreme-tail subset (those beyond $c_{\max}$),
not simply the most poorly-quantised weights by volume.
Preserving the wrong $2\%$ (Voronoi boundary weights) actively hurts quality
relative to leaving them quantised, because their exact values create a
destructive mixed representation with the surrounding centroid-quantised majority.
This finding motivates the percentile-based selection criterion in Algorithm~\ref{alg:legolm}:
distance to nearest centroid is the correct proxy for importance at small $K$,
not weight magnitude or per-layer sensitivity.

\begin{figure}[h]
  \centering
  \includegraphics[width=.9\linewidth]{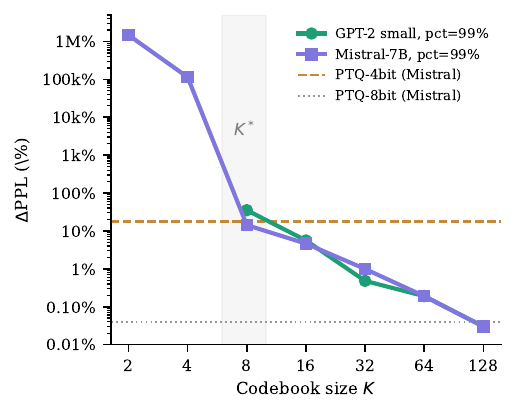}
  \caption{Phase transition in quantisation robustness.
    $\Delta$PPL vs $K$ at $99\%$ on GPT-2 small and Mistral-7B (log scale).
    A sharp transition occurs between $K{=}4$ (catastrophic) and $K{=}8$ (functional).
    This is explained by Theorem~\ref{thm:outlier}.}
  \label{fig:phase_transition}
  \vspace{-2em}
\end{figure}

\subsection{Global vs.\ Per-Layer Codebooks at Scale}
\label{sec:global_local}

Deep Compression~\citep{han2016deep} applies independent per-layer K-means,
giving each layer its own $K$-entry codebook.
On GPT-2 small, this is competitive with \LegoLM{} at matched $K$
(Deep Compr. $K{=}256$: $+10.2\%$ vs \LegoLM{} $K{=}128$ full: $+23.1\%$).
On Mistral-7B, per-layer K-means fails at all tested $K$
($K{=}256$: $+930{,}306\%$).

The failure occurs because Mistral-7B's 32 layers have widely varying weight
distributions: early attention projections have std $\approx 0.020$ while
later output projections have std $\approx 0.007$.
A per-layer codebook with only 256 entries must independently cover each
layer's distribution, leaving many weights as outliers within that layer.
A global codebook implicitly allocates more centroids to the denser,
more common weight values across all layers, achieving better coverage.
This advantage grows with model depth and diversity.

\textbf{Controlled ablation.}
Adding pct$=99\%$ selective replacement to per-layer K-means (Deep Compr.\
$K{=}128$, pct$=99\%$) yields $+0.05\%$ at $4.57{\times}$ on Mistral-7B ---
essentially matching \LegoLM{} $K{=}128$, pct$=99\%$ ($+0.03\%$, $4.41{\times}$).
This isolates the contribution of each component: \emph{selective replacement
is the dominant mechanism; global vs.\ per-layer codebook structure is a
secondary effect contributing $\leq 0.02$ pp quality difference at matched $K$}.
The advantage of a global codebook is therefore practical rather than qualitative:
one codebook fits all 32 layers with $128$ shared centroids vs.\ $32{\times}128
= 4{,}096$ per-layer centroids, and K-means is fitted once rather than per layer.

\subsection{Huffman Coding}

Replacing fixed $\lceil\log_2 K\rceil$-bit indices with a Huffman code
over empirical centroid frequencies yields 4--12\% additional CR improvement
(Table~\ref{tab:main_results}, H-CR column).
The gain is largest at small $K$ (K=8: $+9.9\%$) where centroid usage is most
non-uniform, and smallest at large $K$ (K=128: $+3.7\%$) where usage is nearly flat.
Implementation requires only a standard Huffman encoder applied to the index stream.
\vspace{-1em}

\subsection{Boundary-Layer Sensitivity and Its Scale Dependence}
\label{sec:skipfl}

Table~\ref{tab:skip_fl_ablation} compares \LegoLM{} with and without boundary-layer
protection (\texttt{skip\_fl}) across GPT-2 small (124M) and Mistral-7B (7.2B).

\begin{table}[h]
\centering
\small
\caption{Effect of boundary-layer protection (\texttt{skip\_fl}) on GPT-2 small and
  Mistral-7B. Full replacement ($p{=}100\%$) and selective ($p{=}80\%$) shown.
  Skipping 2 boundary blocks costs 12--21\% additional CR.}
\label{tab:skip_fl_ablation}
\setlength{\tabcolsep}{4pt}
\begin{tabular}{llcrr}
\toprule
\textbf{Model} & \textbf{Variant} & \textbf{CR} & \textbf{No skip} & \textbf{Skip\_fl} \\
\midrule
\multirow{2}{*}{GPT-2 s}
  & $K{=}64$,  full        & 5.33$\times$ & $+38.2\%$           & $+18.4\%$ \\
  & $K{=}128$, full        & 4.57$\times$ & $+23.1\%$           & $+10.1\%$ \\
\midrule
\multirow{3}{*}{Mistral}
  & $K{=}64$,  full        & 5.33$\times$ & $+601{,}635\%$      & $+394{,}441\%$ \\
  & $K{=}128$, full        & 4.57$\times$ & $+1{,}087{,}456\%$  & $+382{,}115\%$ \\
  & $K{=}128$, $80\%$  & 2.67$\times$ & $-0.02\%$           & $-0.01\%$ \\
\bottomrule
\end{tabular}
\vspace{-1em}
\end{table}

The results reveal that boundary-layer protection is \emph{scale-limited}.
On GPT-2 small, \texttt{skip\_fl} provides meaningful rescue: $K{=}64$ full replacement
improves from marginal ($+38\%$) to functional ($+18\%$), and $K{=}128$ full
replacement improves from damaged ($+23\%$) to near-lossless ($+10\%$).

On Mistral-7B, \texttt{skip\_fl} cannot rescue catastrophic failure ---
both variants remain non-functional regardless of boundary protection.
In the lossless regime ($K{=}128$, $p{=}80\%$), \texttt{skip\_fl} provides zero
benefit ($-0.02\% \to -0.01\%$) while costing 10\% additional CR.

This divergence clarifies the relationship between the two failure modes.
Selective replacement (pct) addresses \emph{outlier dominance}, which grows
superlinearly with scale (Table~\ref{tab:scale_law}) and is the dominant failure at 7B.
Boundary-layer protection addresses a secondary sensitivity of the first and last
transformer blocks, which is meaningful at 124M but negligible relative to
outlier dominance at 7B.
We therefore recommend \texttt{skip\_fl} for models below ${\sim}$500M parameters
and rely exclusively on selective replacement at larger scale.

\section{Conclusion}
\label{sec:conclusion}

We have presented a systematic study of weight sharing as a compression paradigm
for LLMs, identifying two distinct failure modes with formal proofs and proposing
\LegoLM{} as their joint solution.


\LegoLM{} is entirely data-free, training-free, and completes in under
30 minutes on a single GPU for Mistral-7B.
We view the failure mode taxonomy and the outlier dominance scaling law as
independently useful contributions to the compression community,
applicable beyond weight sharing to quantization and pruning methods
that currently treat all weights uniformly.


\bibliography{legolm}

\end{document}